\documentclass{llncs}
\usepackage{caption}
\usepackage{subcaption}

\usepackage[numbers]{natbib}
\usepackage{multicol}
\usepackage{xcolor}
\usepackage[bookmarks=true]{hyperref}
\hypersetup{
     colorlinks   = true,
     linkcolor    = blue!50!black,
     citecolor    = red!80!black,
     urlcolor    = blue!30!black
}

\usepackage[]{amsmath}
\usepackage[cmex10]{mathtools}
\usepackage[]{amssymb}
\usepackage{bbm}
\usepackage{tikz}
\usepackage{tcolorbox}

\newcommand{\OPT}{\mathsf{OPT}}
\newcommand{\Plus}{\texttt{+}}
\newcommand{\Minus}{\texttt{-}}
\newcommand{\Cplus}{\CCC^\Plus}
\newcommand{\Cminus}{\CCC^\Minus}
\newcommand{\espPlus}{\mathrm{ESP}^\Plus}

\newcommand{\espMinus}{\mathrm{ESP}^\Minus}
\newcommand{\logtree}{\mathsf{tree}}

\newcommand{\EEi}{{{\EE}_{\normalfont\text{init}}}}
\newcommand{\GGi}{{\GG}_{\normalfont\text{init}}}
\newcommand{\LLi}{{\LL}_{\normalfont\text{init}}}
\newcommand{\Rset}{\mathbb{R}}

\newcommand{\GG}{\mathcal{G}}

\newcommand{\VV}{\mathcal{V}}
\newcommand{\TT}{\mathcal{T}}
\newcommand{\EE}{{\mathcal{E}}}
\newcommand{\xx}{\mathbf{x}}

\newcommand{\LL}{\mathbf{L}}

\newcommand{\MM}{\mathbf{M}}

\newcommand{\ES}{\mathbf{S}}

\newcommand{\WW}{\mathbf{W}}

\newcommand{\ppp}{\boldsymbol\pi}
\newcommand{\AAA}{\mathbf{A}}
\newcommand{\AAAT}{\mathbf{A}^{\hspace{-0.07cm}\top}}
\newcommand{\Acal}{\mathcal{A}}

\newcommand{\Bcal}{\mathcal{B}}
\newcommand{\aaa}{\mathbf{a}}
\newcommand{\ee}{\mathbf{e}}

\newcommand{\II}{\mathbf{I}}

\newcommand{\CCC}{\mathcal{C}}

\DeclareMathAlphabet\mathbfcal{OMS}{cmsy}{b}{n}

\newcommand{\diag}{\mathop{\mathrm{diag}}}

\DeclareFontFamily{OT1}{pzc}{}
\DeclareFontShape{OT1}{pzc}{m}{it}{<-> s * [1.2] pzcmi7t}{}
\DeclareMathAlphabet{\mathpzc}{OT1}{pzc}{m}{it}

\usepackage{algorithm}
\usepackage{algpseudocode}

\algrenewcommand\textproc{}%


\usepackage{eucal}

\usepackage{enumitem}
\usepackage{mdframed}
\usepackage{icomma}
\usepackage{wrapfig}
\begin{document}
\mainmatter              
\title{Designing Sparse
  Reliable Pose-Graph SLAM:\\A Graph-Theoretic Approach}
\titlerunning{Near-$t$-Optimal Graphs}  
\author{Kasra Khosoussi\inst{1} \and Gaurav S.~Sukhatme\inst{2}\\[0.1cm]%
 Shoudong Huang\inst{1} \and Gamini Dissanayake\inst{1}}
\authorrunning{Kasra Khosoussi et al.} 

\institute{Centre for Autonomous Systems\\ University of Technology
  Sydney\\Sydney, NSW 2007, Australia\\
  \email{kasra.khosoussi@uts.edu.au}\\[0.2cm]
\and
Department of Computer Science\\
University of Southern California\\
Los Angeles, CA 90089, USA}

\maketitle              

\begin{abstract}
In this paper, we aim to design sparse D-optimal (determinant-optimal)
pose-graph SLAM problems through the synthesis of sparse graphs with the maximum
weighted number of spanning trees. Characterizing graphs with the maximum number
of spanning trees is an open problem in general. To tackle this problem, several
new theoretical results are established in this paper, including the monotone
log-submodularity of the weighted number of spanning trees. By exploiting these
structures, we design a complementary pair of near-optimal efficient
approximation algorithms with provable guarantees. Our theoretical results are
validated using random graphs and a publicly available
pose-graph SLAM dataset.
\end{abstract}

\section{Introduction}
Graphs arise in modelling numerous phenomena across science and engineering.
In
particular, estimation-on-graph (EoG) is a class of (maximum likelihood) estimation problems with a
natural graphical representation that arise especially in robotics and sensor
networks. In such problems, each vertex corresponds to an unknown state, and
each edge corresponds to a relative noisy measurement between the
corresponding states. Simultaneous localization and mapping (SLAM) and sensor
network localization (SNL) are two well-studied EoGs.

Designing sparse, yet ``well-connected'' graphs is a subtle task that frequently
arises in various domains.  First, note that graph sparsity---in EoGs and many
other contexts---lead to
computational efficiency. Hence, maintaining sparsity is often crucial.
It is useful to see
graph connectivity as a spectrum, as we often need to compare the connectivity
of connected graphs.  In engineering, well-connected graphs often exhibit
desirable qualities such as \emph{reliability}, \emph{robustness}, and
\emph{resilience} to noise, outliers, and link failures.
More specifically, a well-connected EoG is more resilient to a fixed noise level and
results in a more reliable estimate (i.e., smaller estimation-error covariance
in the Loewner ordering sense).
Consequently,
maintaining a sufficient connectivity is also critical.  
Needless to say,
sparsity is, by its very essence, at odds with well-connectivity.
This is the case in SLAM, where there is a trade-off between the cost of
inference and the reliability of the resulting estimate.
This problem is not new.
Measurement selection and pose-graph pruning
have been extensively studied in the SLAM literature (see, e.g.,
\cite{6094414,huang2013consistent}). However, in this paper we take a novel
graph-theoretic approach by reducing the problem of designing sparse reliable
SLAM problems to the purely combinatorial problem of synthesizing sparse, yet
well-connected graphs. 
In what follows, we briefly justify this
approximate reduction.

First, note that by estimation reliability we refer to the
standard D-optimality criterion, defined as the determinant of the (asymptotic)
maximum likelihood estimator covariance matrix. D-optimality is a standard
and popular design criterion; see, e.g., \cite{Joshi2009,Khosoussi2014} and the
references therein. Next, we have to
specify how we measure graph connectivity.
Among the
existing combinatorial and spectral graph connectivity criteria, the number of spanning trees (sometimes
referred to as \emph{graph complexity} or \emph{tree-connectivity}) stands out: despite its
combinatorial origin, it can also be characterized solely by the spectrum of the
graph Laplacian \cite{Godsil2001}.
In \cite{Khosoussi2014,Khosoussi2015graph,kasra16icra}, we shed light on the
connection between the D-criterion in SLAM---and some other EoGs---and
the tree-connectivity of the underlying graph. Our theoretical and empirical
results demonstrate that, under some standard conditions, D-optimality in SLAM
is \emph{significantly} influenced by the tree-connectivity of the graph
underneath. Therefore, one can accurately estimate the D-criterion without using any
information about the robot's trajectory or realized measurements (see
Section~\ref{sec:robot}). Intuitively speaking, our approach can be seen as a
dimensionality reduction scheme for designing D-optimal SLAM problems from the
joint space of trajectories and graph topologies to only the space of graph
topologies \cite{kasra16icra}.

Although this work is specifically motivated by the SLAM problem, 
designing sparse graphs with the maximum tree-connectivity has several other
important applications.
For example,
it has been shown that tree-connectivity is associated with the
D-optimal incomplete block designs \cite{gaffke1982d,cheng1981maximizing,bailey2009combinatorics}.
Moreover, tree-connectivity is a major factor in maximizing the connectivity of
certain random graphs that model unreliable networks under random link failure
(\emph{all-terminal network reliability})
\cite{kelmans1983multiplicative,weichenberg2004high}. In particular, a classic
result in network reliability theory states that  if the \emph{uniformly-most reliable}
network exits, it must have the maximum tree-connectivity among all graphs with
the same size \cite{bauer1987validity,myrvold1996reliable,boesch2009survey}.  
\vspace{-14pt}
\subsubsection*{Known Results.}
Graphs with the maximum weighted number of spanning trees among a family of
graphs with the same vertex set are called
\emph{$t$-optimal}. The problem of
characterizing unweighted $t$-optimal graphs among the set of graphs with $n$
vertices and $m$ edges remains open and has been solved \emph{only} for specific
pairs of $n$ and $m$; see, e.g.,
\cite{shier1974maximizing,cheng1981maximizing,kelmans1996graphs,petingi2002new}.
The span of these special cases is too narrow for the types of graphs that
typically arise in SLAM and sensor networks. Furthermore, in many cases the
$(n,m)$ constraint alone is insufficient for describing the true set of
“feasible” graphs and cannot capture implicit practical constraints that exist
in SLAM. Finally, it is not clear how these results can be extended to the
case of (edge) weighted graphs, which are essential for representing SLAM
problems, where the weight of each edge represents the precision of the
corresponding pairwise measurement \cite{kasra16icra}.
\vspace{-15pt}
\subsubsection*{Contributions.}
This paper addresses the problem of designing sparse $t$-optimal graphs with the ultimate
goal of designing D-optimal pose-graph SLAM problems. First and foremost, we
formulate a combinatorial optimization problem that captures the measurement
selection and measurement pruning scenarios in SLAM. Next, we prove that the weighted
number of spanning trees, under certain conditions, is a monotone log-submodular function
of the edge set. To the best of our knowledge, this is a new result in graph theory.
Using this result, we prove that the greedy algorithm is near-optimal. In our
second approximation algorithm, we
formulate this problem as an integer program that admits a straightforward convex relaxation.
Our analysis sheds light on the performance of a simple deterministic rounding
procedure that have also been used in more general contexts.
The proposed approximation algorithms provide near-optimality certificates.
The proposed graph synthesis framework can be
readily applied to any application
where maximizing tree-connectivity is desired.

\vspace{-15pt}
\subsubsection*{Notation.}
Throughout this paper, bold lower-case and upper-case letters are reserved for
vectors and matrices, respectively.
The standard basis for $\mathbb{R}^{n}$ is denoted by
$\{\ee_{i}^n\}_{i=1}^{n}$.
Sets are shown by upper-case letters.
$|\cdot|$ denotes the set cardinality. 
For any finite set $\mathcal{W}$, $\binom{\mathcal{W}}{k}$ is the set of all
$k$-subsets of $\mathcal{W}$. We use $[n]$ to denote the set $\{1,2,\dots,n\}$.
The eigenvalues of symmetric matrix $\MM$ are denoted by
\mbox{$\lambda_1(\MM) \leq \dots \leq \lambda_n(\MM)$}. $\mathbf{1}$, $\II$ and
$\mathbf{0}$ denote the vector of all ones, the identity and the zero matrices with appropriate sizes,
respectively.  $\ES_1\succ\ES_2$ (resp. $\ES_1 \succeq \ES_2$) means $\ES_1 -
\ES_2$ is positive definite (resp. positive semidefinite).
Finally, $\diag(\WW_1,\dots,\WW_k)$ is
the block-diagonal matrix whose main diagonal blocks are $\WW_1,\dots,\WW_k$.

\vspace{-0.1cm}
\section{Preliminaries}
\label{sec:pre}
\subsubsection*{Graph Matrices.} Throughout this paper, we usually refer to undirected graphs
$\GG = (\VV,\EE)$ with $n$ vertices (labeled with $[n]$) and $m$ edges. With a
little abuse of notation, we call $\widetilde{\AAA} \in \{-1,0,1\}^{n \times m}$
the incidence matrix
of $\GG$ after choosing an arbitrary orientation for its edges. The Laplacian
matrix of $\GG$ is defined as $\widetilde{\LL} \triangleq
\widetilde{\AAA}\widetilde{\AAA}^\top$.
$\widetilde{\LL}$ can be written as $\sum_{i=1}^{m} \widetilde{\LL}_{e_i}$ in
which $\widetilde{\LL}_{e_i}$ is the \emph{elementary Laplacian} associated with
edge $e_i = \{u_i,v_i\}$, where the $(u_i,u_i)$ and $(v_i,v_i)$ entries are $1$, and
the $(u_i,v_i)$ and $(v_i,u_i)$ entries are $-1$.
{Anchoring} 
$v_0 \in \VV$ is
equivalent to removing the row associated with $v_0$ from $\widetilde{\AAA}$.
Anchoring $v_0$ results in the \emph{reduced}  incidence matrix
$\AAA$ and the \emph{reduced}  Laplacian matrix $\LL \triangleq
\AAA \AAAT$. $\LL$ is also known as the \emph{Dirichlet}.
We may assign positive weights to the edges of $\GG$ via $w : \EE \to
\mathbb{R}_{>0}$. Let $\WW \in \mathbb{R}^{m \times m}$ be the diagonal
matrix whose $(i,i)$ entry is equal to the weight of the $i$th edge.
The \emph{weighted} Laplacian (resp. reduced \emph{weighted} Laplacian) is then defined as
$\widetilde{\LL}_w \triangleq \widetilde{\AAA}\WW\widetilde{\AAA}^\top$ (resp.
$\LL_w \triangleq \AAA\WW\AAA^{\hspace{-0.1cm}\top}$). Note that the (reduced) unweighted
Laplacian is a special case of the (reduced) weighted Laplacian with $\WW =
\II_m$ (i.e., when all edges have unit weight).
\vspace{-0.3cm}
\subsubsection*{Spanning Trees.}
A spanning tree of ${\GG}$ is a spanning subgraph of ${\GG}$ that is also a
tree. Let $\mathbb{T}_{\GG}$ denote the
set of all spanning trees of $\GG$. $t(\GG)\triangleq
|\mathbb{T}_{\GG}|$ denotes the number of spanning trees in $\GG$. As a
generalization, for graphs
whose edges are weighted by $w : \EE \to \mathbb{R}_{>0}$, we define
the \emph{weighted number of spanning trees}, 
\begin{equation}
  t_w(\GG) \triangleq \sum_{\TT \in \mathbb{T}_{\GG}}
  \mathbb{V}_{\hspace{-0.05cm}w}(\TT).
\end{equation}
We call $\mathbb{V}_{\hspace{-0.05cm}w} : \mathbb{T}_\GG \to \mathbb{R}_{>0}$
the \emph{tree value function} and define it as the product of the edge
weights along a spanning tree. Notice that for unit edge weights, $t_w(\GG)$
coincides with $t(\GG)$. Thus, unless explicitly stated otherwise, we generally
assume the graph is weighted. To prevent overflow and underflow, it is more
convenient to work with $\log t_w(\GG)$. We formally define
\emph{tree-connectivity} as,
\begin{equation}
  \tau_{w}(\GG) \triangleq 
  \begin{cases}
	\log t_w(\GG) & \text{if $\GG$ is connected,} \\
	0           & \text{otherwise.}
  \end{cases}
  \label{}
\end{equation}
For the purpose of this work, without loss
of generality we can assume $w(e) \geq 1$ for all $e \in
\EE$, and thus $\tau_w(\GG) \geq 0$.\footnote{Replacing any $w : \EE \to
  \mathbb{R}_{\geq 0}$ with $w^\prime : \EE \to
  \mathbb{R}_{\geq 1} : e \mapsto
\alpha_w w(e)$ for a sufficiently large constant $\alpha_w$ does not affect
the set of $t$-optimal graphs.} 
The equality occurs only when either
$\GG$ is not connected, or when $\GG$ is a tree whose all edges have unit
weight.
Kirchhoff's seminal matrix-tree theorem is a classic result in spectral graph theory.
This theorem relates the spectrum of the Laplacian matrix of graph to its number of spanning
trees. The original matrix-tree theorem states that,
\begin{align}
  t(\GG) & = \det \LL  \\ & =
  \frac{1}{n} 
  \prod_{i=2}^{n} \lambda_i(\widetilde{\LL}).
\end{align}
Here $\LL$ is the reduced Laplacian after anchoring an arbitrary
vertex.
Kirchhoff's matrix-tree theorem has been generalized to the case of
edge-weighted graphs. According to the generalized theorem,
$t_w(\GG)  = \det \LL_w =
\frac{1}{n}
\prod_{i=2}^{n} \lambda_i(\widetilde{\LL}_w)$.
\subsubsection*{Submodularity.}
  Suppose $\mathcal{W}$ is a finite set.
  Consider a set function $\xi : 2^{\mathcal{W}} \to \mathbb{R}$. $\xi$ is
  called:
  \begin{enumerate}
    \item \emph{normalized} \,iff $\xi(\varnothing) = 0$.
    \item \emph{monotone} \,\,\,\,iff $\xi(\Bcal) \geq \xi(\Acal)$ for every
      $\Acal$ and $\Bcal$ s.t. $\Acal \subseteq \Bcal \subseteq \mathcal{W}$.
    \item \emph{submodular} iff 
      for every
      $\Acal$ and $\Bcal$ s.t. $\Acal \subseteq \Bcal \subseteq \mathcal{W}$ and
      \mbox{$\forall s \in
	\mathcal{W}
      \setminus \Bcal$}
      we have,
      \begin{equation}
	\xi(\Acal \cup \{s\}) - \xi(\Acal) \geq \xi(\Bcal \cup \{s\}) -
	\xi(\Bcal).
	\label{}
      \end{equation}
  \end{enumerate}
\section{D-Optimality via Graph Synthesis}
\label{sec:robot}
In this section, we discuss the connection between D-optimality and
$t$-optimality in SLAM by briefly reviewing the results in 
\cite{Khosoussi2014,Khosoussi2015graph,kasra16icra}. 
Consider the 2-D pose-graph SLAM problem where each measurement consists of the
rotation (angle) and translation between a pair of robot poses over time, 
corrupted by an independently-drawn additive zero-mean Gaussian noise. According
to our model, the
covariance matrix of the noise
vector corrupting the $i$th measurement can be written as
$\diag(\sigma_{p_i}^{2}\II_2,\sigma_{\theta_i}^{2})$, where $\sigma_{p_i}^2$
and $\sigma_{\theta_i}^{2}$ denote the translational and
rotational noise variances, respectively. As mentioned earlier, SLAM, as an EoG
problem, admits a natural graphical representation $\GG = (\VV,\EE)$ in which
poses correspond to graph vertices and edges correspond to the relative
measurements between the corresponding poses. Furthermore, 
measurement precisions are incorporated into our model by assigning positive
weights to the edges of $\GG$. Note that for each edge we have two separate weight functions
$w_p$ and $w_\theta$, defined as $w_p : e_i \mapsto \sigma_{p_i}^{-2}$ and
$w_{\theta} : e_i \mapsto \sigma_{\theta_i}^{-2}$.

Let $\mathbb{V}\mathrm{ar}[\hat{\xx}_{\textsf{mle}}]$ be the asymptotic
covariance matrix of the maximum likelihood
estimator (Cram\'{e}r-Rao lower bound) for estimating the trajectory $\xx$.
In \cite{Khosoussi2014,Khosoussi2015graph,kasra16icra}, we investigated the
impact of graph topology on the \mbox{D-optimality} criterion ($\det
\mathbb{V}\mathrm{ar}[\hat{\xx}_{\textsf{mle}}]$) in SLAM. 
The results presented in \cite{kasra16icra} are threefold. 
First, in \cite[Proposition 2]{kasra16icra} it is proved that
\vspace{-4pt}
\begin{equation}
  -2\,\tau_{w_p}(\GG) - \log\det(\LL_{w_\theta}+\gamma\II) \leq
  \log\det\mathbb{V}\mathrm{ar}[\hat{\xx}_{\textsf{mle}}] \leq
  -2\,\tau_{w_p}(\GG) - \tau_{w_\theta}(\GG)
\vspace{-3pt}
\end{equation}
in which $\gamma$ is a parameter whose value depends on the maximum distance between
the neighbouring  robot poses normalized by $\sigma_{p_i}^2$'s; e.g., this
parameter shrinks by reducing the distance between the neighbouring poses, or by
reducing the precision of the translational measurements (see
\cite[Remark 2]{kasra16icra}). Next, based on this result, it is easy to see
that \cite[Theorem
5]{kasra16icra}, 
\vspace{-3pt}
\begin{equation}
  \lim_{\gamma \to 0^+} \log\det\mathbb{V}\mathrm{ar}[\hat{\xx}_{\textsf{mle}}]
  = - 2\,\tau_{w_p}(\GG) - \tau_{w_\theta}(\GG).
  \label{eq:SLAMth}
\vspace{-3pt}
\end{equation}
Note that the expression above depends only on the graphical
representation of the problem.
Finally, the empirical observations and Monte Carlo simulations based on a
number of synthetic and real datasets indicate that the RHS of \eqref{eq:SLAMth}
provides a reasonable estimate for
$\log\det\mathbb{V}\mathrm{ar}[\hat{\xx}_{\textsf{mle}}] $ even in the
non-asymptotic regime where $\gamma$ is not negligible. 
In what follows, we demonstrate how these results can be used in a
graph-theoretic approach to the D-optimal measurement selection and pruning
problems.
\vspace{-0.1cm}
\subsubsection*{Measurement Selection.} 
Maintaining sparsity is essential for computational efficiency in
SLAM, especially in long-term autonomy.
Sparsity can be preserved by
implementing a measurement selection policy to asses the significance of new or
existing measurements. Such a vetting process can be realized by (i)
assessing the significance of any new measurement before adding it to the graph,
and/or (ii) pruning a subset of the acquired measurements if their
contribution is deemed to be insufficient. These ideas have been
investigated in the literature; for the former approach see, e.g.,
\cite{Joshi2009,shamaiah2010greedy}, and see, e.g.,
\cite{6094414,huang2013consistent} for the latter.

Now consider the D-optimal measurement selection problem whose goal is to select
the optimal $k$-subset of measurements such that the resulting
$\log\det\mathbb{V}\mathrm{ar}[\hat{\xx}_{\textsf{mle}}]$ is minimized.  This
problem is closely related to the D-optimal sensor selection problem for which
 two successful approximation algorithms have been proposed in \cite{Joshi2009}
and \cite{shamaiah2010greedy} under the assumption of linear sensor models.
The measurement models in SLAM are nonlinear. Nevertheless, we can still use
\cite{Joshi2009,shamaiah2010greedy} after linearizing the measurement model.
Note that the Fisher
information matrix and $\log\det\mathbb{V}\mathrm{ar}[\hat{\xx}_{\textsf{mle}}]$ in SLAM depend
on the true $\xx$. Since the true value of $\xx$ is not available, in practice
these terms are approximated by evaluating the Jacobian matrix at the estimate
obtained by maximizing the log-likelihood function using an iterative solver.

An alternative approach would be to replace
$\log\det\mathbb{V}\mathrm{ar}[\hat{\xx}_{\textsf{mle}}]$ with a
graph-theoretic objective function based on  \eqref{eq:SLAMth}. 
Note that this is equivalent to reducing the original problem into a graph synthesis problem.  
The graphical approach has the following
advantages:
\begin{enumerate}
  \item \emph{Robustness}: Maximum likelihood estimation in SLAM boils down to
	solving a non-convex
	optimization problem via iterative solvers. These solvers are subject to local
	minima. 
	Hence, the approximated
	$\log\det\mathbb{V}\mathrm{ar}[\hat{\xx}_{\textsf{mle}}]$ can be highly inaccurate
	and lead to misleading designs if the Jacobian is evaluated at a local
	minimum (see \cite[Section
	VI]{kasra16icra} for an example). The graph-theoretic objective function
	based on \eqref{eq:SLAMth}, however, is independent of the trajectory $\xx$
	and, therefore, is robust to such convergence errors.
	\vspace{0.1cm}
  \item \emph{Flexibility}: To directly compute
	$\log\det\mathbb{V}\mathrm{ar}[\hat{\xx}_{\textsf{mle}}]$, we first need a
	nominal or estimated trajectory $\xx$. Furthermore, for the latter we also
	need to know the realization of relative measurements. Therefore, any design
	or decisions made in this way will be confined to a particular trajectory.
	On the contrary, the graphical approach requires only the knowledge of the
	topology of the graph, and thus is more flexible. Note that the $t$-optimal
	topology corresponds to a range of trajectories. Therefore, the graphical
	approach enables us to assess the D-optimality of a particular design with
	minimum information and without relying on any particular---planned, nominal or
	estimated---trajectory.
\end{enumerate}
We will
investigate the problem of designing $t$-optimal graphs in
Section~\ref{sec:Synthesis}.

\vspace{-0.15cm}
\section{Synthesis of Near-$t$-Optimal Graphs}
\label{sec:Synthesis}
\subsubsection*{Problem Formulation.}
In this section, we formulate and tackle the combinatorial optimization problem
of designing sparse graphs with the maximum weighted tree-connectivity.
Since the decision variables are the edges of the graph, it is more convenient
to treat the weighted tree-connectivity as a function of the edge set of the
graph for a given set of vertices ($\VV = [n]$) and a positive weight function
$w : \binom{[n]}{2} \to \Rset_{\geq 1}$. $\logtree_{n,w}:2^{\binom{[n]}2} \to
\Rset_{\geq 0} : \EE \mapsto \tau_w([n],\EE)$ takes
as input a set of edges $\EE$ and returns the weighted tree-connectivity of
graph $([n],\EE)$.
To simplify our notation, hereafter we drop $n$ and/or $w$ from
$\logtree_{n,w}$ (and similar terms) whenever $n$ and/or $w$ are clear from the context.
\begin{tcolorbox}[colback=blue!0!white,colframe=black!90]
  \begin{problem}[$k$-ESP]
	Suppose the following are given:
	\begin{itemize}
	  \item[$\bullet$] a \emph{base graph} $\GGi = ([n],\EEi)$
	  \item[$\bullet$] a weight function $w : \binom{[n]}{2} \to
		\Rset_{\geq 1}$
	  \item[$\bullet$] a set of $c$ \emph{candidate} edges (either $\Cplus$ or
		$\Cminus$)
	  \item[$\bullet$] an integer $k \leq c$ 
	\end{itemize}
	Consider the following edge selection problems (ESP):
	\begin{itemize}
	  \item[$\diamond$] $k$-$\espPlus{}$
		\begin{equation}
		  \begin{aligned}
			& \underset{\EE \subseteq \Cplus \subseteq \binom{[n]}{2} \setminus \EEi}{\text{maximize}}
			& & \logtree(\EEi \cup \EE) \quad \text{\normalsize subject to} \quad |\EE| = k.
		  \end{aligned}
		  \label{eq:addEdge}
		\end{equation}
	  \item[$\diamond$] $k$-$\espMinus{}$
		\begin{equation}
		  \begin{aligned}
			& \underset{\EE \subseteq \Cminus \subseteq  \EEi}{\text{maximize}}
			& & \logtree(\EEi \setminus \EE) \quad \text{\normalsize subject to}
			\quad |\EE| = k.
		  \end{aligned}
		  \label{eq:delEdge}
		\end{equation}
	\end{itemize}
  \end{problem}
\end{tcolorbox}
\begin{remark}
  It is easy to see that any instance of \eqref{eq:addEdge} can be expressed as an instance of
  \eqref{eq:delEdge} and vice versa. 
  Therefore, without loss
  of generality, in this work we only consider $k$-$\espPlus$.
\end{remark}
\subsubsection*{$1$-$\espPlus$.}Consider the simple case of $k = 1$. $\Delta_{uv} \triangleq \aaa_{uv}
\LL^{-1} \aaa_{uv}$ is known as the \emph{effective
resistance} between vertices $u$ and $v$. Here $\aaa_{uv} \in
\{-1,0,1\}^{n-1}$ is the vector
$\ee_u^n - \ee_v^n$ after crossing out the entry that corresponds to the anchored
vertex.
Effective resistance has emerged
from several other contexts as a key factor; see, e.g., \cite{ghosh2008minimizing}.
In \cite[Lemma 3.1]{kasraArxiv16} it is shown that the optimal choice in
$1$-$\espPlus$ is the candidate edge with the maximum $w(e)\Delta_{e}$. The
effective resistance can be efficiently computed by performing a Cholesky
decomposition on the reduced weighted Laplacian matrix of the base graph $\LLi$ and
solving a triangular linear system (see \cite{kasraArxiv16}). In the worst case
and for a dense base graph $1$-$\espPlus$ can be solved in $O(n^3 + c \,n^2)$ time. 
\subsection{Approximation Algorithms for $k$-$\espPlus$}
Solving the general case of $k$-$\espPlus$ by exhaustive search requires
examining $\binom{c}{k}$ graphs. This is not practical even when $c$ is bounded
(e.g., for $c=30$ and $k=10$ we need to perform more than $3 \times 10^7$ Cholesky
factorizations). Here we propose a complementary pair of approximation algorithms.
\subsubsection{I: Greedy.}
The greedy algorithm finds an approximate solution to $k$-$\espPlus$ by
decomposing it into a sequence of $k$ $1$-$\espPlus$ problems, each of which can be solved
using the procedure outlined above.
After solving each
subproblem, the optimal edge is moved from the candidate set to the base graph. The next $1$-$\espPlus$ subproblem is defined using the updated
candidate set and the updated base graph. If the graph is dense, a naive implementation of the greedy
algorithm requires less than $O(k c n^3)$ operations. An efficient
implementation of this approach that requires $O(n^3 + kcn^2)$ time is described in \cite[Algorithm 1]{kasraArxiv16}.
\vspace{-15pt}
\subsubsection*{Analysis.} 
Let $\GGi = ([n],\EEi)$ be a \emph{connected} base graph and $w : \binom{[n]}{2} \to
\Rset_{\geq 1}$. Consider the following function.
\begin{align}
  \mathcal{X}_{w}: \EE  \mapsto \logtree({\EE \cup \EEi})  - \logtree({\EEi}).
  \label{}
\end{align}
In $k$-$\espPlus$, we restrict the domain of $\mathcal{X}_{w}$ to $2^{\Cplus}$.
Note that $\logtree({\EEi})$ is a constant and, therefore, we can express the
objective function in $k$-$\espPlus$ using $\mathcal{X}_w$,
\begin{equation}
  \begin{aligned}
	& \underset{\EE \subseteq \Cplus}{\text{maximize}}
	& & \mathcal{X}_{w}(\EE) \quad \text{\normalsize subject to} \quad |\EE| = k.
  \end{aligned}
  \label{eq:kPlus}
\end{equation}
\vspace{-15pt}
  \begin{theorem}
	$\mathcal{X}_{w}$ is normalized, monotone and submodular.
	\label{th:logTG-sub}
  \end{theorem}
\begin{proof}
  Omitted due to space limitation---see the technical report
  \cite{kasraArxiv16}.
\end{proof}
\vspace{-5pt}
Maximizing an arbitrary monotone submodular function subject to a cardinality
constraint {can be} NP-hard in general (see, e.g., the Maximum Coverage
problem \cite{hochbaum1996approximation}). 
Nemhauser et al. \cite{nemhauser1978analysis} in their seminal work have shown that the greedy
algorithm is a constant-factor approximation algorithm with a factor of
$\eta \triangleq (1-1/e) \approx 0.63$ for any
(normalized) monotone submodular function subject to a cardinality
constraint. 
Let $\OPT$ be the optimum value of \eqref{eq:addEdge},
$\EE_\text{greedy}$ be the edges selected by the greedy algorithm, $\tau_{\text{greedy}} \triangleq
\logtree(\EE_{\text{greedy}}\cup\EEi)$ and $\tau_{\text{init}} \triangleq
\logtree(\EEi)$.
  \begin{corollary}
	$\normalfont
	\tau_{\text{greedy}} \geq \eta\,\OPT + (1-\eta)\,\tau_{\text{init}}$.
  \end{corollary}

\vspace{-15pt}
\subsubsection{II: Convex Relaxation.}
In this section, we design an approximation algorithm
for $k$-$\espPlus$ through convex relaxation. We begin by
assigning an auxiliary variable $0 \leq \pi_i \leq 1$ to each candidate edge
$e_i \in\Cplus$.
The idea is to reformulate the problem such that finding the optimal set of
candidate edges is equivalent to finding the optimal
value for $\pi_i$'s. 
Let $\ppp \triangleq [\pi_1 \,\, \pi_2 \,\, \cdots \,\, \pi_c]^\top$ be the
stacked vector of auxiliary variables. 
Define,
\begin{equation}
    \LL_w(\ppp) \triangleq \LLi + 
      \sum_{\mathclap{e_i \in \Cplus}} \pi_{i} w(e_i) \LL_{e_i} = \AAA
      \WW^\pi\hspace{-0.09cm} \AAAT,
  \label{eq:Lpi}
\end{equation}
where $\LL_{e_i}$ is the reduced elementary 
Laplacian, $\AAA$ is the reduced incidence matrix of $\GG_{\bullet} \triangleq
([n],\EEi \cup \Cplus)$, and $\WW^\pi$ is the diagonal matrix of edge weights
assigned by the following weight function,
\begin{equation}
  w^\pi(e_i) = 
  \begin{cases}
    \pi_i w(e_i) & e_i \in \Cplus, \\
    w(e_i) & e_i \notin \Cplus.
  \end{cases}
  \label{}
\end{equation}
\vspace{-10pt}
\begin{lemma}
  If
  $\GGi$ is connected, $\LL_w(\ppp)$ is positive definite for any $\ppp \in
  [0,1]^{c}$.
\end{lemma}
As before, for convenience we assume $\GGi$ is connected.
Consider the following optimization problems over $\ppp$.

\noindent\begin{minipage}{.5\linewidth}
\begin{equation}
  \begin{aligned}
    & \underset{\ppp}{\text{maximize}}
    & & \log\det {\LL_w(\ppp)}\\
    & \text{subject to}
    && \|\ppp\|_0 = k,\\
    &&& 0\leq \pi_i \leq {1}, \, \forall i \in [c].
  \end{aligned}
  \label{eq:conv0}
  \tag{P$_1$}
\end{equation}
\end{minipage}
\quad\hfill\quad
\begin{minipage}{.5\linewidth}
\begin{equation}
  \begin{aligned}
    & \underset{\ppp}{\text{maximize}}
    & & \log\det {\LL_w(\ppp)}\\
    & \text{subject to}
    && \|\ppp\|_1 = k,\\
    &&& \pi_i \in \{0,1\}, \, \forall i \in [c].
  \end{aligned}
  \label{eq:conv0bin}
  \tag{P$^\prime_1$}
\end{equation}
\end{minipage}
\vspace{0.2cm}

\noindent\ref{eq:conv0} is equivalent to our original definition of
$k$-$\espPlus$. 
First, note that from the generalized matrix-tree theorem we know that the objective function is
equal to the weighted tree-connectivity of graph $\GG_\bullet = ([n],\EEi \cup \Cplus)$ whose
edges are weighted by $w^\pi$.
The auxiliary variables act as selectors: the $i$th candidate edge is
selected iff $\pi_i = 1$.
 The
combinatorial difficulty of $k$-$\espPlus$ here is embodied in the
non-convex $\ell_0$-norm constraint. 
It is easy to see that in \ref{eq:conv0}, at the optimal solution, auxiliary variables
take binary values.
This is why the integer program \ref{eq:conv0bin} is equivalent to \ref{eq:conv0}.
A natural choice for relaxing \ref{eq:conv0bin} is to replace $\pi_i \in \{0,1\}$
with $0 \leq \pi_i \leq 1$; i.e.,
\begin{equation}
  \begin{aligned}
    & \underset{\ppp}{\text{maximize}}
    & & \log\det {\LL_w(\ppp)}\\
    & \text{subject to}
    && \|\ppp\|_1 = k,\\
    &&& 0\leq \pi_i \leq {1}, \, \forall i \in [c].
  \end{aligned}
  \label{eq:conv1}
  \tag{P$_2$}
\end{equation}
The feasible set of \ref{eq:conv1} contains that of \ref{eq:conv0bin}. Hence, the optimum value of \ref{eq:conv1} is an upper bound for the optimum
of \ref{eq:conv0} (or, equivalently, \ref{eq:conv0bin}).  Note that the
$\ell_1$-norm constraint here is identical to $\sum_{i=1}^{c} \pi_i = k$.
\ref{eq:conv1} is a convex optimization problem since the objective function
(tree-connectivity) is concave and the constraints are linear and affine in $\ppp$.
In fact, \ref{eq:conv1} is an instance of the $\mathrm{MAXDET}$ problem
\cite{vandenberghe1998determinant} subject to additional affine 
constraints on $\ppp$.
It is worth noting that \ref{eq:conv1} can be reached also by relaxing the non-convex $\ell_0$-norm
constraint in \ref{eq:conv0} into the convex $\ell_1$-norm constraint
$\|\ppp\|_1
= k$.
Furthermore, \ref{eq:conv1} is also closely related to a $\ell_1$-regularised
variant of $\mathrm{MAXDET}$,
\begin{equation}
  \begin{aligned}
    & \underset{\ppp}{\text{maximize}}
    & & \log\det {\LL_w(\ppp)} - \lambda \, \|\ppp\|_1\\
    & \text{subject to}
    && 0\leq \pi_i \leq {1}, \, \forall i \in [c].
  \end{aligned}
  \label{eq:conv1b}
  \tag{P$_3$}
\end{equation}
This problem is a penalized form of \ref{eq:conv1}; these two problems
are equivalent for some positive value of $\lambda$.
Problem~\ref{eq:conv1b} is also a convex optimization problem for any non-negative
$\lambda$.
The $\ell_1$-norm in \ref{eq:conv1b} penalizes the loss of sparsity, while the
log-determinant rewards stronger tree-connectivity.  
$\lambda$ is a parameter that controls the sparsity of the resulting graph; i.e.,
a larger $\lambda$ yields a sparser vector of selectors $\ppp$. \ref{eq:conv1b}
is closely related to graphical lasso
\cite{friedman2008sparse}.
\ref{eq:conv1} (and \ref{eq:conv1b}) can be solved globally in polynomial time using
interior-point methods \cite{Boyd2004,Joshi2009}.  After finding a globally optimal
solution $\ppp^\star$ for the relaxed problem \ref{eq:conv1}, we ultimately need
to map it into a feasible $\ppp$ for \ref{eq:conv0bin}; i.e., choosing $k$ edges
from the candidate set $\Cplus$. 
\vspace{-5pt}
  \begin{lemma}
	$\ppp^\star$ is an optimal solution for
	$k$-$\normalfont \espPlus$ iff 
	 $\ppp^\star \in \{0,1\}^c$.
  \end{lemma}
\vspace{-18pt}
\subsubsection*{Rounding.} In general, $\ppp^\star$ may contain
fractional values that need to be mapped into feasible integral values for
\ref{eq:conv0bin} by a \emph{rounding procedure} that sets $k$ auxiliary
variables to one and others to zero. The most intuitive deterministic rounding
policy is to pick
the $k$ edges with the largest $\pi^\star_i$'s. 

The idea behind the convex relaxation technique described so far can be seen as a
graph-theoretic special case
of the algorithm proposed in \cite{Joshi2009}. However, it is not
clear yet how the solution of the relaxed convex problem
\ref{eq:conv1} is related to that of the original non-convex $k$-$\espPlus$ in the
integer program \ref{eq:conv0bin}.
To answer this question, consider the following randomized strategy. We may attempt to
find a suboptimal solution for $k$-$\espPlus$ by randomly sampling
candidates. In this case, for the $i$th candidate edge, we flip a coin whose
probability of heads is $\pi_i$ (independent of other candidates). We then
select that candidate edge if the coin lands on head.  
\begin{theorem} Let the
  random variables $k^{\ast}$ and $t_w^\ast$ denote, respectively, the number of
  chosen candidate edges and the corresponding weighted number of spanning trees
  achieved by the above randomized algorithm. Then,
	\begin{align}
	  \mathbb{E}\,[k^\ast] & = \sum_{i=1}^{c} \pi_i, \\
	  \mathbb{E}\,[t_w^\ast] & = \det\LL_w(\ppp).
	\end{align}
	\label{th:random}
  \end{theorem}
\begin{proof}
  See \cite{kasraArxiv16} for the proof.\footnote{A generalized version of
  this theorem that covers the more
  general case of \cite{Joshi2009} is proved in \cite{kasraArxiv16}.}
\end{proof}
According to Theorem~\ref{th:random}, the randomized
algorithm described above on average selects $\sum_{i=1}^c \pi_i$ candidate edges and  
achieves $\det\LL_w(\ppp)$ weighted number of spanning trees in expectation. 
Note that these two terms appear in the constraints and objective of the
relaxed problem \ref{eq:conv1}, respectively. Therefore, the relaxed problem can
be interpreted as the problem of finding the optimal sampling probabilities
$\ppp$ for
the randomized algorithm described above.
This offers a new narrative:
\begin{corollary}
  The objective in \ref{eq:conv1} is to find
  the optimal probabilities $\ppp^\star$ for sampling edges from $\normalfont\Cplus$ such that the
  weighted number of spanning trees is maximized in \emph{expectation}, while the
  \emph{expected} number of newly selected edges is equal to $k$.
\end{corollary}
In other words, \ref{eq:conv1} can be seen as a convex relaxation of
\ref{eq:conv0} at the expense of maximizing the objective and satisfying the
constraint, both \emph{in expectation}.
This new interpretation can be used as a basis for designing randomized rounding
procedures based on the randomized technique described above.  If one uses
$\ppp^\star$ (the fractional solution of the relaxed problem \ref{eq:conv1}) in
the aforementioned randomized rounding scheme, Theorem~\ref{th:random} ensures
that, on  average, such a method attains $\det\LL(\ppp^\star)$ by picking $k$
new edges in expectation.  Finally, we note that this new interpretation sheds
light on why the
deterministic rounding policy described earlier performs well in practice. Note
that randomly sampling
candidate edges with the probabilities in $\ppp^\star$ does not necessarily
result in a
feasible solution for \ref{eq:conv0bin}. That being said, consider every feasible outcome in which
exactly $k$ candidate edges are selected by the randomized algorithm with
probabilities in $\ppp^\star$. It is
easy to show that the deterministic procedure described earlier (picking
$k$ candidates with the largest $\pi^\star_i$'s) is in fact selecting the most probable
feasible outcome (given that exactly $k$ candidates have been selected).
\vspace{-0.3cm}
\subsubsection{Near-Optimality Certificates.}
It is impractical to
compute $\OPT$ via exhaustive search in large problems. Nevertheless, the approximation
algorithms described above yield lower and upper bounds for $\OPT$ that can be
quite tight in practice. Let $\tau^\star_\text{cvx}$ be the
optimum value of \ref{eq:conv1}. Moreover, let $\tau_\text{cvx}$ be the
suboptimal value obtained after rounding the solution of \ref{eq:conv1} (e.g., picking
the $k$ largest $\pi_i^\star$'s). The following corollary readily follows from
the analysis of the greedy and convex approximation algorithms.
  \begin{corollary}
	\label{cor:bound}
	\begin{equation}
	  \normalfont
	  \max\,\Big\{ \tau_\text{greedy},\tau_\text{cvx} \Big\}
	  \leq
	  \OPT
	  \leq
	  \min\,\Big\{ \mathcal{U}_\text{greedy},\tau^\star_\text{cvx} \Big\}
	  \label{eq:optbound}
	\end{equation}
	where $\mathcal{U}_\text{\normalfont greedy} \triangleq
	\zeta\tau_\text{\normalfont greedy} +
	  (1-\zeta)\tau_\text{\normalfont init}$ in which $\zeta \triangleq \eta^{-1} \approx 1.58$. 
  \end{corollary}
The bounds in Corollary~\ref{cor:bound} can be computed by running the greedy
and convex relaxation algorithms. Whenever $\OPT$ is beyond reach, the upper
bound can be
used to asses the quality of any feasible design.
Let $\mathcal{S}$ be an arbitrary $k$-subset of $\Cplus$ and $\tau_\mathcal{S}
\triangleq \logtree(\mathcal{S} \cup \EEi)$. $\mathcal{S}$ can be, e.g., the solution of
greedy algorithm, the solution of \ref{eq:conv1} after rounding, an existing
design (e.g., an existing pose-graph problem) or a suboptimal solution proposed
by a third party. Let $\mathcal{L}$ and $\mathcal{U}$
denote the lower and upper bounds in \eqref{eq:optbound}, respectively. From
Corollary~\ref{cor:bound} we have,
\begin{equation}
  \max \, \Big\{0,\mathcal{L}-\tau_{\mathcal{S}}\Big\} \leq \underbrace{\OPT -
	\tau_\mathcal{S}}_{\text{optimality gap}} \leq \mathcal{U} -
	\tau_\mathcal{S}.
  \label{}
\end{equation}
Therefore, $\mathcal{U} - \tau_\mathcal{S}$ (or similarly,
$\mathcal{U}/\tau_\mathcal{S} \geq \OPT/\tau_\mathcal{S}$) can be used as a
near-optimality certificate for an arbitrary design $\mathcal{S}$.
\vspace{-0.3cm}
\subsubsection{Two Weight Functions.} In the synthesis problem studied so far,
it was implicitly assumed that each edge is weighted by a single weight
function. This is not necessarily the case in SLAM, where each measurement has two components, each of
which has its own precision, i.e., $w_p$ and $w_\theta$
in \eqref{eq:SLAMth}.
Hence, we need to revisit the
synthesis problem in a more general setting, where multiple weight
functions assign weights, simultaneously, to a single edge. 
It turns out that the proposed approximation algorithms and their analyses can
be easily generalized to handle this case.
\begin{enumerate}[leftmargin=*]
  \item \emph{Greedy Algorithm}:
	For the greedy algorithm, we just need to replace $\mathcal{X}_{w}$ with
	$\mathcal{Y}_{w} : \EE \mapsto 2 \, \mathcal{X}_{w_p}(\EE) +
	\mathcal{X}_{w_\theta}(\EE)$; see
	\eqref{eq:SLAMth}. Note
	that $\mathcal{Y}_{w}$ is a linear combination of normalized monotone submodular
	functions with positive weights, and therefore is also normalized, monotone
	and submodular.
  \item \emph{Convex Relaxation}:
 The convex relaxation technique can be generalized to the
case of multi-weighted edges by replacing the concave objective function
$\log\det\LL_w(\ppp)$ with $2\,\log\det\LL_{w_p}(\ppp) +
\log\det\LL_{w_\theta}(\ppp)$, which is also concave.

\end{enumerate}

\begin{remark}
  Recall that our goal was to design sparse, yet reliable SLAM problems.
  So far we considered the problem of designing D-optimal SLAM problems with a
  given number of edges. The dual approach would be to find the sparsest SLAM problem such that
  the determinant of the estimation-error covariance is less than a
  desired threshold.
  Take for example the following scenario: find the sparsest SLAM problem by
  selecting loop-closure measurements from a given set of candidates such that
  the resulting D-criterion is $50\%$ smaller than that of dead reckoning. The
  dual problem can be written as,
  \begin{equation}
	\begin{aligned}
	  & \underset{\EE \subseteq \Cplus}{\text{minimize}}
	  & & |\EE| \quad \text{\normalsize subject to} \quad \mathcal{X}_{w}(\EE)
	  \geq \tau_{\textsf{min}}.
	\end{aligned}
	\label{eq:dual1}
  \end{equation}
  in which $\tau_{\textsf{min}}$ is given. In \cite{kasraArxiv16} we have shown that our proposed
  approximation algorithms and their analyses can be easily modified to address the dual
  problem. Due to space limitation, we have to refrain from discussing the dual
  problem in this paper.
\end{remark}
\vspace{-18pt}
\subsection{Experimental Results}
\begin{figure}[t]
  \centering
  \begin{subfigure}[t]{0.48\textwidth}
    \includegraphics[width=\textwidth]{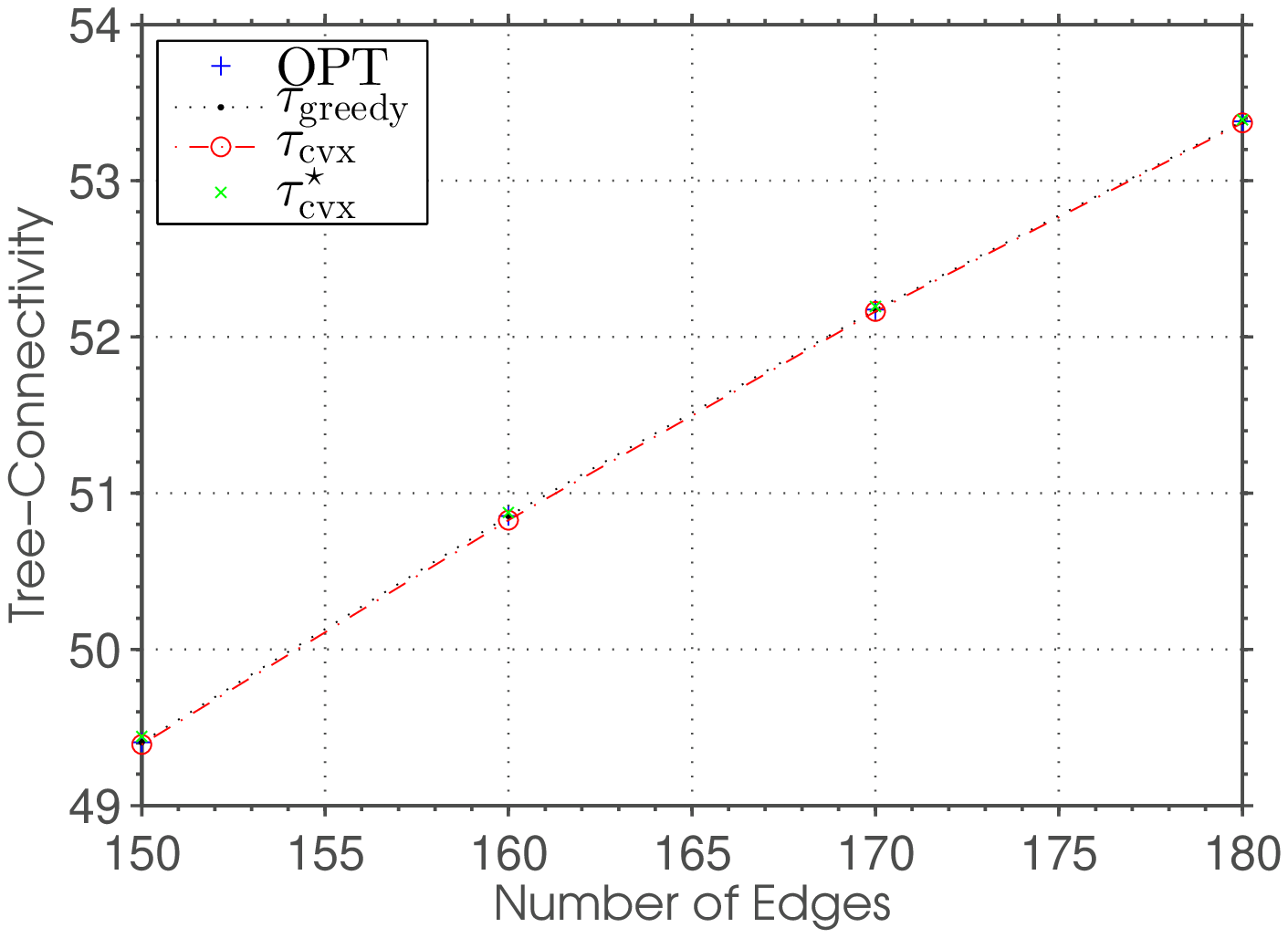}
    \caption{\small $n = 20$, $k = 5$}
      \label{fig:varEdgesOpt}
    \end{subfigure}
    ~
  \begin{subfigure}[t]{0.48\textwidth}
    \includegraphics[width=\textwidth]{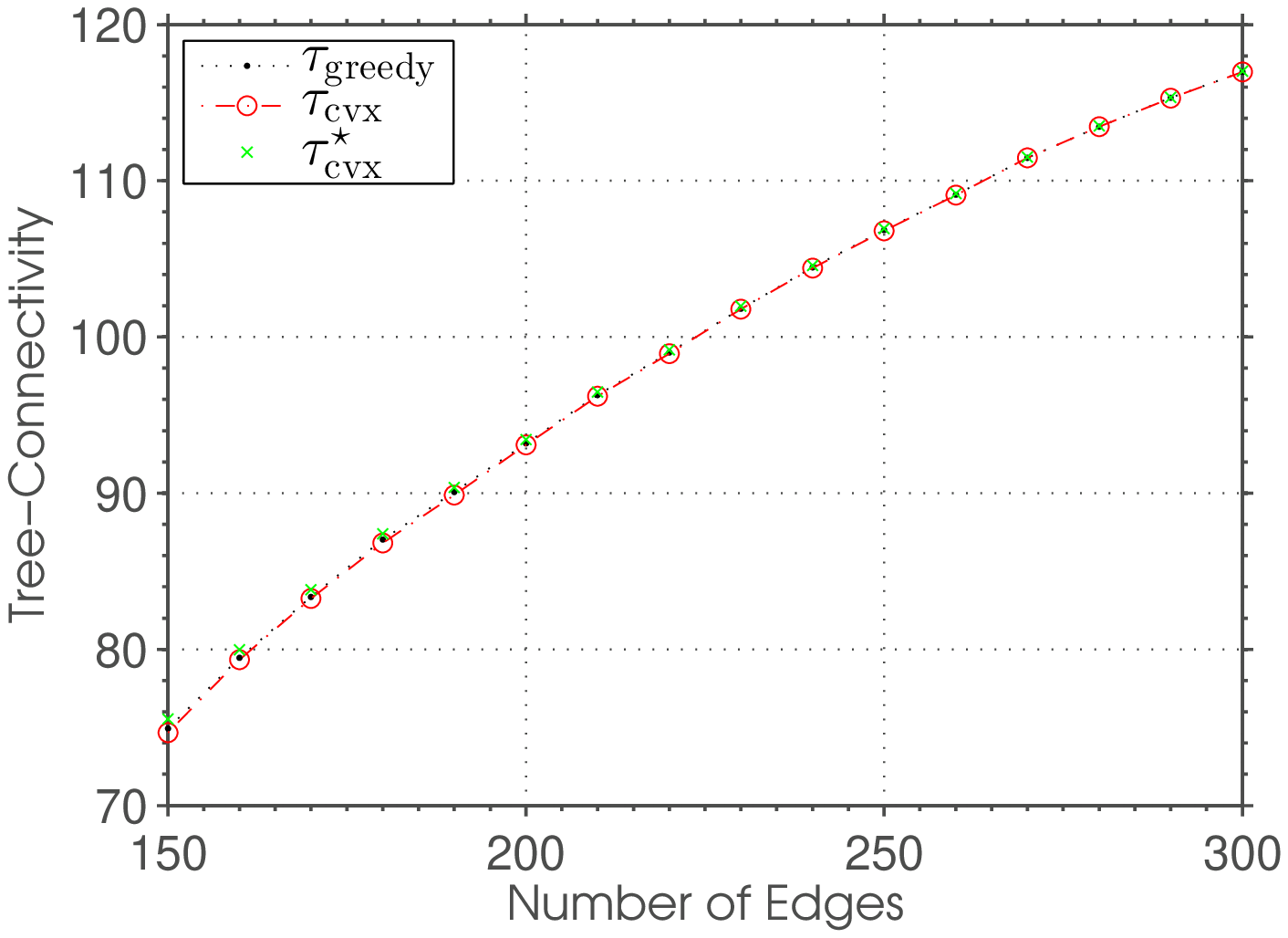}
    \caption{\small $n=50$, $k = 5$}
    \label{fig:varEdges}
  \end{subfigure}
  \\
  \begin{subfigure}[t]{0.5\textwidth}
    \includegraphics[width=\textwidth]{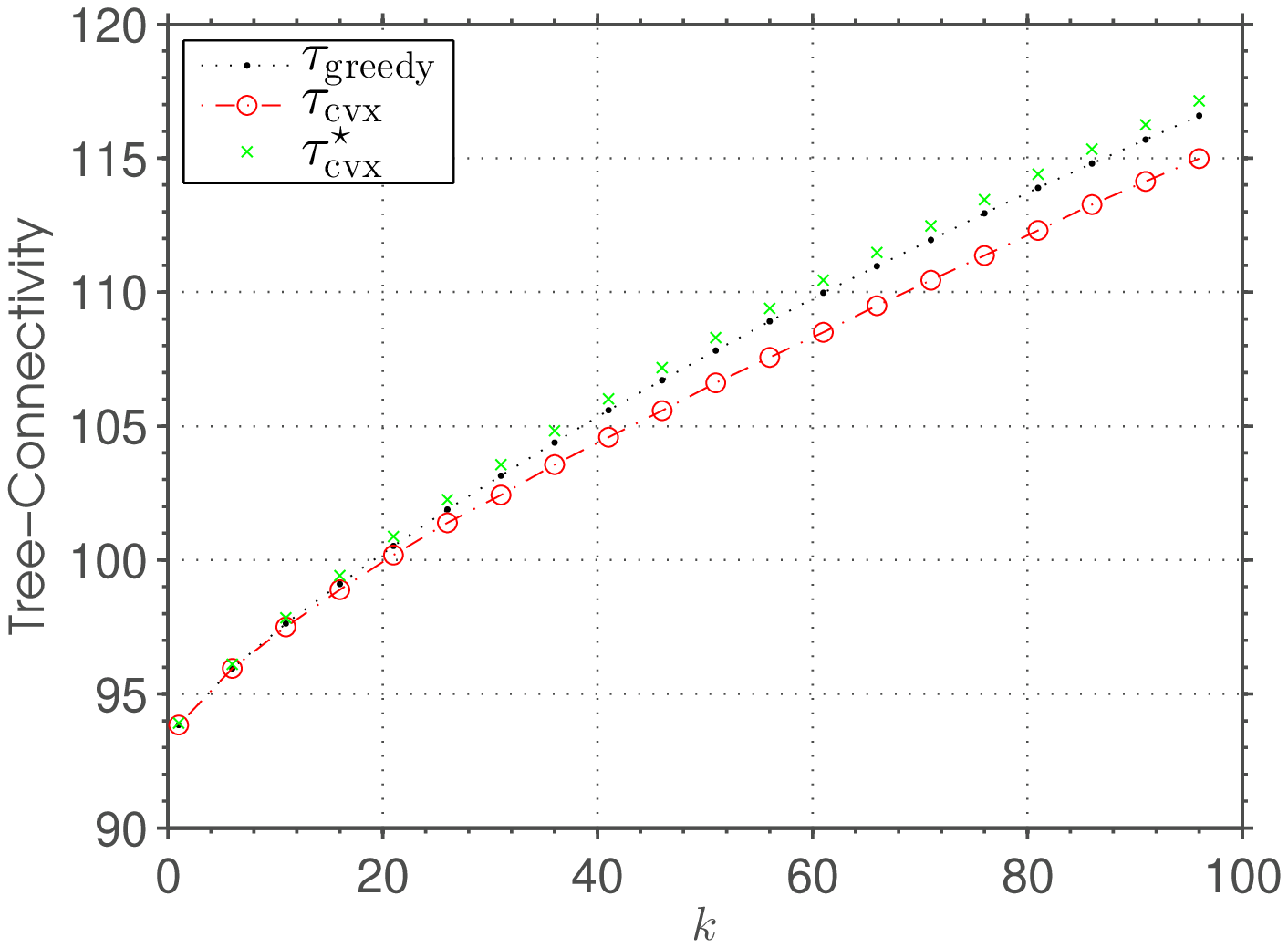}
    \caption{\small $n = 50$, $|\EEi| = 200$}
    \label{fig:varK}
  \end{subfigure}
  \caption{\small $k$-$\espPlus$ on randomly generated graphs with $\Cplus =
	\binom{[n]}{2} \setminus \EEi$.}
  \label{fig:esps}
\end{figure}
The proposed algorithms were implemented in MATLAB.
\ref{eq:conv1}
is modelled using YALMIP \cite{YALMIP} and solved
using SDPT$3$ \cite{tutuncu2003solving}. 
\subsubsection{Random Graphs.} Figure~\ref{fig:esps} illustrates the performance of our approximate
algorithms in randomly generated graphs. 
The set of candidates in these experiments is $\Cplus = \binom{[n]}{2} \setminus \EEi$.
Figures~\ref{fig:varEdgesOpt} and \ref{fig:varEdges} show the resulting tree-connectivity as a
function of the number of randomly generated edges for a fixed $k=5$ and,
respectively, $n = 20$ and $n = 50$. 
Our results indicate that both algorithms exhibit
remarkable performances for $k = 5$. Note that computing $\mathsf{OPT}$ by exhaustive search is only feasible in
small instances such as Figure~\ref{fig:varEdgesOpt}. Nevertheless,
computing the exact $\mathsf{OPT}$ is not crucial for evaluating our approximate
algorithms, as Corollary~\ref{cor:bound} guarantees that $ \tau^{\star}_{\text{greedy}}\leq \OPT \leq
{\tau}^\star_{\text{cvx}}$;
i.e., the space between each black $\mathrm{\cdot}$ and the corresponding
\textcolor{green!60!black}{green} $\times$.
Figure~\ref{fig:varK} shows the results obtained for varying $k$. The optimality
gap for $\tau_\text{cvx}$ gradually grows as the planning horizon $k$
increases.  Our greedy algorithm, however, still yields a near-optimal
approximation. 
\subsubsection{Real Pose-Graph Dataset.}
We also evaluated the proposed algorithms on the Intel Research Lab dataset as a popular
pose-graph SLAM benchmark.\footnote{\url{https://svn.openslam.org/data/svn/g2o/trunk/data/2d/intel/intel.g2o}} In this scenario, $\EEi$ is chosen to be the set of
odometry edges, and $\Cplus$ is the set of loop closures. The parameters in this
graph are $n = 943$, $|\EEi| = 942$ and $|\Cplus| = 895$. Note that computing the true
$\OPT$ via exhaustive search is clearly impractical; e.g., for $k=100$, there are more
than $10^{134}$ possible graphs. For the edge weights, we are using the original
information (precisions) reported in the dataset. Since the translational and
rotational measurements have different precisions, two weight functions---$w_p$
and $w_\theta$---assign weights to each edge of the graph, and the objective is
to maximize $2\,\tau_{w_p}(\GG) +
\tau_{w_\theta}(\GG)$. Figure~\ref{fig:intel} shows the resulting objective
value for the greedy and convex relaxation approximation algorithms, as well as
the upper bounds ($\mathcal{U}$) in Corollary~\ref{cor:bound}.\footnote{See also
  \url{https://youtu.be/5JZF2QiRbDE} for a visualization.}
According to
Figure~\ref{fig:intel}, both algorithms have successfully found near-$t$-optimal
(near-D-optimal) designs. The greedy algorithm has outperformed the convex relaxation
with the simple deterministic (sorting) rounding procedure. For small values of
$k$, the upper bound $\mathcal{U}$ on $\OPT$ is given by
$\mathcal{U}_\text{greedy}$ (\textcolor{blue}{blue} curve). However, for
$k \geq 60 $, the convex relaxation provides a significantly tighter upper bound on
$\OPT$ (\textcolor{green!60!black}{green} curve).
In this dataset, YALMIP+SDPT3 on an
Intel Core i5-2400 operating at 3.1 GHz can solve the convex program in about $20$-$50$
seconds, while a naive implementation of the greedy algorithm (without using
rank-one updates) can solve the case with $k=400$ in about $25$ seconds.

\begin{figure}[t]
  \centering
  \begin{subfigure}[t]{0.48\textwidth}
	\includegraphics[width=\textwidth]{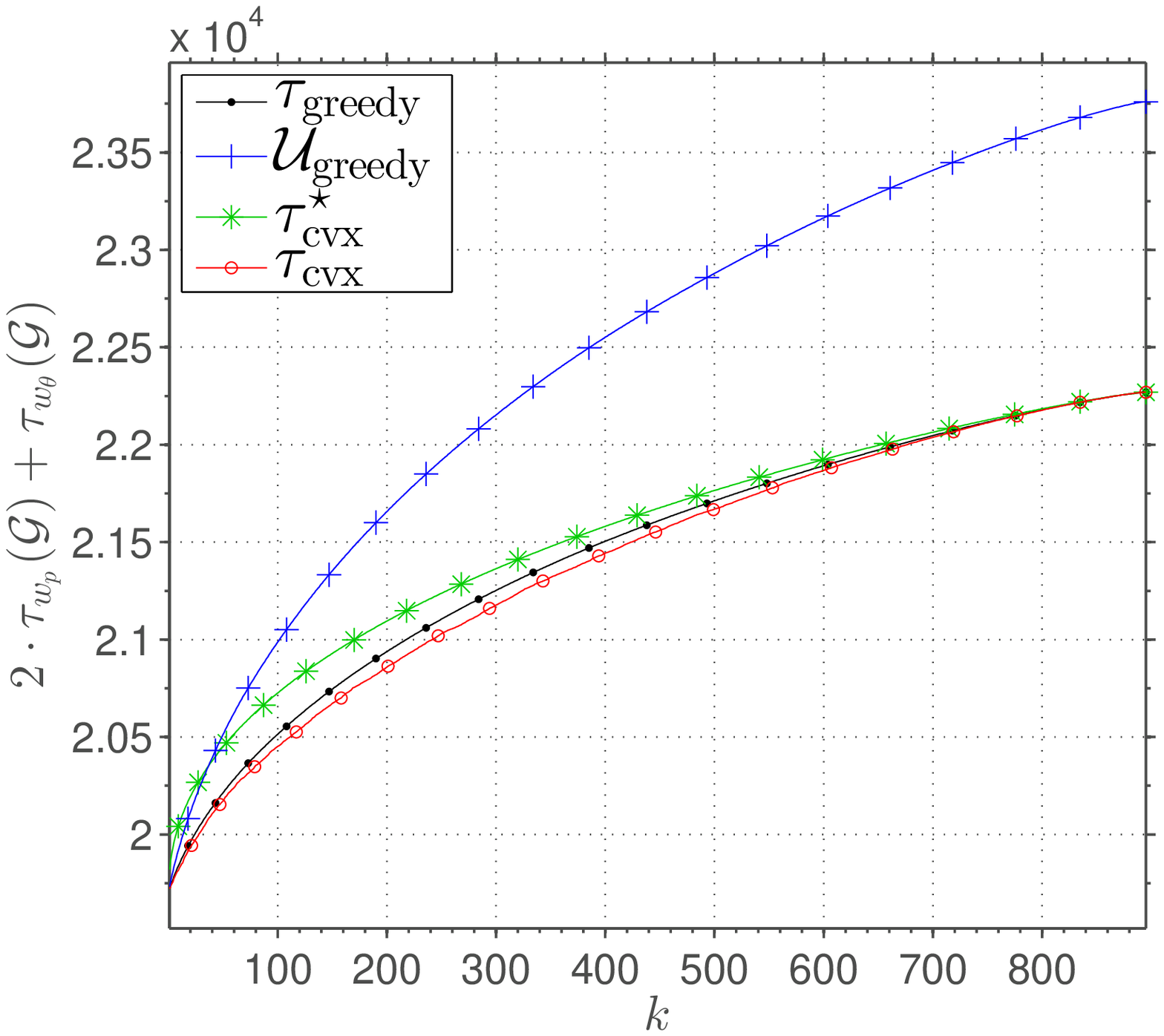}
	\caption{Performance of the proposed approximation algorithms.}
	\label{fig:intel1}
  \end{subfigure}
  ~
  \begin{subfigure}[t]{0.42\textwidth}
	\includegraphics[width=\textwidth]{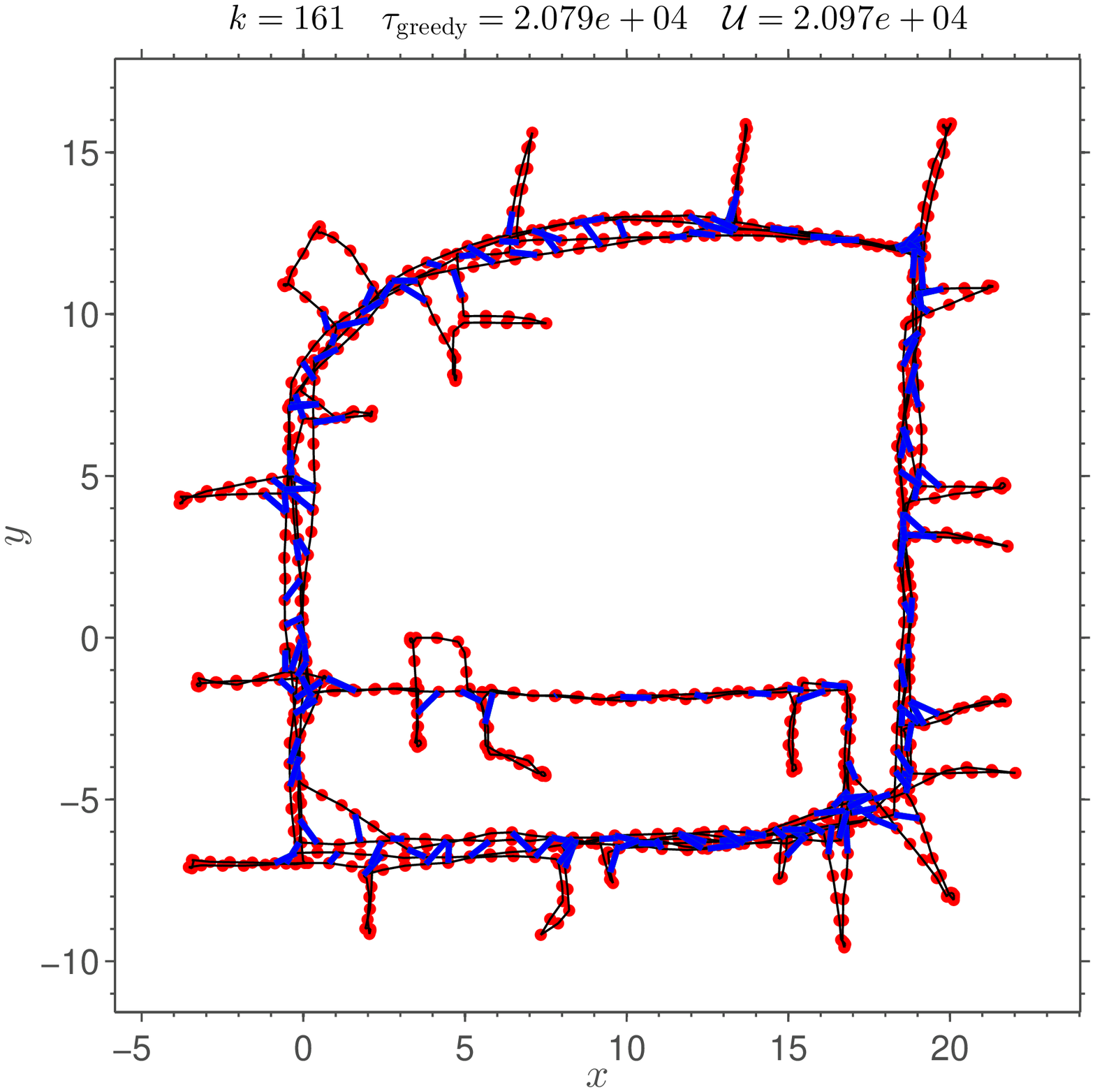}
	\caption{Greedy design for $k=161$ loop closures (out of $895$). Loop-closure
	edges are shown in blue.}
	\label{fig:intel2}
  \end{subfigure}
	\caption{$k$-$\espPlus$ for pose-graph SLAM on the Intel Research Lab dataset.}
	\label{fig:intel}
\end{figure}
\section{Conclusion}
\label{sec:conclusion}
We presented a graph-theoretic approach to the problem 
of designing sparse reliable (i.e., near-D-optimal)
pose-graph SLAM. This paper demonstrated that this problem boils down to a combinatorial
optimization problem whose goal is to find a sparse graph with the maximum
weighted number of spanning trees. The problem of characterizing $t$-optimal graphs
is an open problem with---to the best of our knowledge---no known efficient
algorithm. We designed two efficient approximation algorithms with
provable guarantees and near-optimality certificates.
First and foremost, we introduced a new submodular graph invariant, i.e.,
weighted tree-connectivity. This was used to guarantee that the greedy
algorithm is a constant-factor approximation algorithm for this problem with a
factor of $(1-1/e)$ (up to a constant normalizer). In another approach, we formulated the original
combinatorial optimization problem as
an integer program that admits a natural convex relaxation. We discussed deterministic and
randomized rounding schemes. Our analysis sheds light on the
connection between the original and the relaxed problems.
Finally, we evaluated the performance of the proposed approximation algorithms using random graphs
and a real pose-graph SLAM dataset.
Although this paper specifically targeted SLAM, we note that the proposed algorithms can be
readily used to synthesize near-$t$-optimal graphs in any domain where maximizing
tree-connectivity is useful. See, e.g.,
\cite{chemistryGraph,kim2013network,boesch2009survey,kasraArxiv16} for applications in 
Chemistry, RNA modelling, network reliability under
random link failure and estimation over sensor networks, respectively.
\bibliographystyle{splncs_srt}
\bibliography{graph,slam}
\end{document}